\documentclass[letterpaper]{article}

\usepackage{natbib,alifeconf}  

\usepackage{gensymb }
\usepackage{hyperref}
\usepackage{url}
\usepackage{thmtools, thm-restate}
\usepackage{amsmath}
\usepackage{amssymb}
\usepackage{mathtools}
\usepackage{amsthm}
\usepackage{subfigure}

\newtheorem{theorem}{Theorem}

\newtheorem{lemma}[theorem]{Lemma}

\title{ALIFE2022 template}

\title{Arbitrary Order Meta-Learning with Simple Population-Based Evolution}
\author{Chris Lu, Sebastian Towers \and Jakob Foerster \\
\mbox{}\\
Department of Engineering Sciences, University of Oxford, Oxford, United Kingdom \\
christopher.lu@exeter.ox.ac.uk} 

\begin{document}
\maketitle

\begin{abstract}

Meta-learning, the notion of learning to learn, enables learning systems to quickly and flexibly solve new tasks. This usually involves defining a set of outer-loop \textit{meta-parameters} that are then used to update a set of inner-loop parameters. Most meta-learning approaches use complicated and computationally expensive bi-level optimisation schemes to update these meta-parameters. Ideally, systems should perform multiple orders of meta-learning, i.e. to learn to learn to learn and so on, to accelerate their own learning. Unfortunately, standard meta-learning techniques are often inappropriate for these higher-order meta-parameters because the meta-optimisation procedure becomes too complicated or unstable. Inspired by the higher-order meta-learning we observe in real-world evolution, we show that using simple population-based evolution \textit{implicitly} optimises for arbitrarily-high order meta-parameters. First, we theoretically prove and empirically show that population-based evolution implicitly optimises meta-parameters of arbitrarily-high order in a simple setting. We then introduce a minimal \textit{self-referential} parameterisation, which in principle enables arbitrary-order meta-learning. Finally, we show that higher-order meta-learning improves performance on time series forecasting tasks. 

\end{abstract}

\section{Introduction}

The natural world contains multiple orders of meta-evolution and adaptation \citep{vanchurin2022toward}. For example, DNA has not just evolved to produce an organism, but has also evolved to be \textit{evolvable} \citep{zheng2020selection, woods2011second, metzgar2000evidence}. In other words, DNA has evolved such that random mutations in a genotype frequently result in useful or adaptive changes in the resulting organism's phenotype. Furthermore, the evolution of DNA has created organisms that have the ability to \textit{adapt} within their lifetime, one form of which is organisms that perform \textit{reinforcement learning} \citep{bateson1984genes}. These learning organisms further influence their own learning through social interactions and culture \citep{henrich2015secret, heyes2018cognitive}.

However, most existing works only investigate single-order meta-learning, for example for evolving reinforcement learning algorithms \citep{lu2022discovered}. These approaches commonly use computationally expensive bi-level optimisation schemes that quickly becomes unstable or computationally intractable when applied to higher orders of meta-learning \citep{metz2021gradients}.

Past work has empirically shown that population-based evolution implicitly selects for single order meta-learning, usually by simultaneously evolving mutation rates \citep{frans2021population, back1992self, smith1998self}. Other work has investigated multiple orders of meta-learning, but in the context of gradient-based optimisation \citep{chandra2019gradient} and multi-agent learning \citep{willi2022cola}. Finally, \citet{kirsch2022eliminating, lange2022discovering, metz2021training} empirically investigate using evolution-like algorithms on \textit{self-referential} systems to perform self-referential meta-learning, an idea first articulated in \citet{schmidhuber1987evolutionary}. However, these works  do not theoretically prove that they perform higher-order meta-learning. We connect these works by \textit{theoretically proving} and empirically showing that under some circumstances simple population-based evolution selects for \textit{arbitrarily-high} orders of meta-learning, which in principle allows for arbitrary orders of self-improvement in self-referential systems. 


\section{Numeric Fitness World}

We perform population-based evolution by selecting and mutating the top $k$ most fit individuals at each generation. Unlike past work, we do this on genomes with \textit{multiple orders} of meta-parameters. In particular, we represent a genome $x$ at generation $t$ with $n$ orders of meta-parameters as a vector of $n$ parameters, $x_t = \{x_t^0, x_t^1, \cdots, x_t^n\}$, where $x^i$ represents the $i$th-order meta-parameter. We consider the setting of ``Numeric Fitness World'' in which $\textit{fitness}(x_t) = x^0_t$
, proposed in \citet{frans2021population}.

We \textit{mutate} $x_t$ using the following update rule:
\begin{align}
    && x^i_{t+1} = x^i_{t} + x^{i+1}_t + B_t^i, 0 \leq i < n, 0<t \\
    && x^n_{t+1} = x^n_t + B^n_t, 0<t \\
    && B_t^i \sim \mathcal{N}(0,\beta), i.i.d, \forall t,i
\end{align}

In other words, we update the $i$th-order meta-parameter by adding the $(i+1)$th parameter and noise sampled from a normal distribution. We update the last meta-parameter ($x_t^n$) with just the noise. To instead create a \textit{self-referential} parameterisation, we update the last meta-parameter with \textit{itself} and the noise. 
This exact form of self-reference is likely inappropriate in most settings, but may be sensible in other parameterisations, such as neural networks \citep{irie2022modern}.

\subsection{Theoretical Results}
We prove that top-k selection selects for the fitness of higher-order meta-parameters in this setting if and only if $k>1$.  

Let $P$ define a population of individual members as defined above. Let $x$ be a specific member of $P$. Let $\bar{P}$ define a population identical to $P$ except in the $n$-th parameter of $x$. More specifically, $\bar{x}^n - x^n = \delta$, $\delta > 0$. 

Let $F(P,B,t)$ and $F^{-1}(P,B,t)$ represent the set of fitnesses of the children and non-children of $x$ respectively in population $P$ after $t$ generations of selection and vector of mutations $B$. Note that  $|F(P,B,t+1)|$ would therefore be the number of children of $x$ after generation $t$.

First, we show that top-1 (single-genome) selection does not select for higher-order meta-parameters.

\begin{theorem}
$\mathbb{E}[|F(\bar{P},B,n+1)|] = \mathbb{E}[|F(P,B,n+1)|]$ under top-1 selection for $n>1$.
\end{theorem}

\begin{proof}
    The number of children at generation $t>1$ is entirely determined by the first selection. Either \textit{all} members of the population at generation $n$ are children of $x$, or \textit{none} of them are. As $x^n$ ($n>1$) does not affect the first selection, it is independent to the number of children.
\end{proof}

Next, we show that top-k selection selects for higher-order meta-parameters for $k>1$.

\begin{lemma}
$|F(\bar{P},B,n+1|B=b)| \geq |F(P,B,n+1|B=b)|$ for any vector of mutations $b$.
\end{lemma}

\begin{proof}
    Note that for $t < n$, $F(\bar{P},B,t|B=b) = F(P,B,t|B=b)$, as none of the fitnesses are influenced by $x^n$, the only value in which the two populations differ.

    $F(\bar{P},B,n|B=b)   = F(P,B,n|B=b) \oplus \delta$ where $\oplus$ represents a distributed addition. 
    
    $F^{-1}(\bar{P},B,n|B=b)  = F^{-1}(P,B,n|B=b)$ because $x^n$ can not influence the selection or fitness of other members before generation $n+1$. 

    Thus, there can be no fewer children of $\bar{x}$ than children of $x$ in the top-k of the next generation .
\end{proof}

\begin{theorem}
$\mathbb{E}[|F(\bar{P},B,n+1)|] > \mathbb{E}[|F(P,B,n+1)|]$
\end{theorem}

\begin{proof}
By Lemma 2,  $|F(\bar{P},B,n+1)| \geq |F(P,B,n+1)|$. 
Hence, showing $\mathbb{P}(|F(\bar{P},B,n+1)| > |F(P,B,n+1)|) > 0$ is sufficient for our result. In particular, we show $\mathbb{P}(|F(P,B,n+1)|=0 \cap |F(\bar{P},B,n+1)|=1) >0 $. There is a set of intervals over $B$ such that exactly $k$ members of $F^{-1}(P,B,n)$ lie in the range $[\max F(P,B,n), \max F(P,B,n)+\delta]$ and the rest are less than $\max F(P,B,n)$. Thus, after selection there is exactly one child of $\bar{x}$, and none of $x$.
\end{proof}

\subsection{Empirical Results}

We simulate the evolution using Jax \citep{jax2018github} and show the results in Figure \ref{fig:pbml_self}. We observed that the asymptotic growth in fitness is approximately of the order $x^n$ where $n$ is the number of meta-parameters. Furthermore, the fitness of the self-referential meta-learner grows \textit{exponentially}. Thus, in our population-based setting higher orders of meta-parameters improve fitness. In contrast, in single-genome selection, the expected value of the fitness is largely independent of the number of meta-parameters, demonstrating that single-genome selection does not perform meta-optimisation.

\begin{figure}[t]
\begin{center}
\includegraphics[width=3.2in]{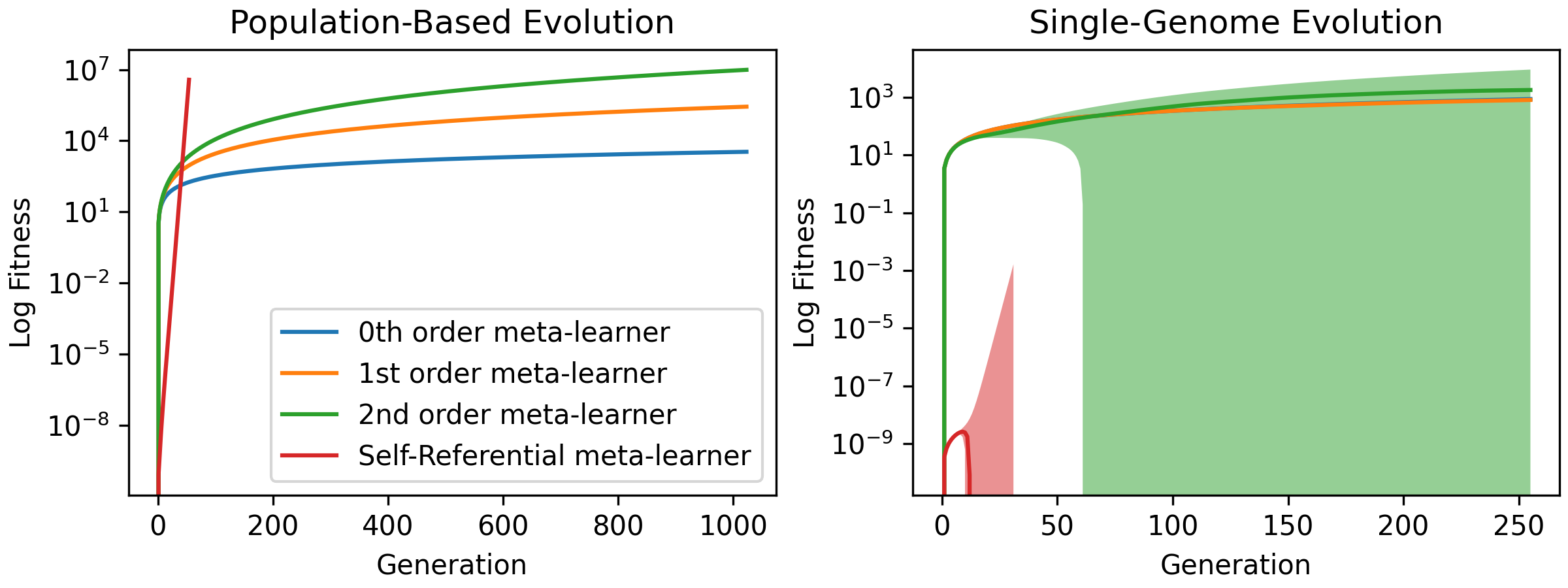}
\caption{Population-based evolution (Top-2) and single-genome evolution with varying orders of meta-learning with a population size of 2048. The shaded region refers to the standard error of the mean across 1024 seeds.}
\label{fig:pbml_self}
\end{center}
\end{figure}

\begin{table}[t]
\begin{tabular}{r|llll}
\multicolumn{1}{l|}{} & \multicolumn{4}{c}{Meta-Learning Order}                    \\
$f(t)$ & \multicolumn{1}{l|}{0th} & \multicolumn{1}{l|}{1st} & \multicolumn{1}{l|}{2nd} & 3rd \\ \hline
$t$                   & 1.0    & \textbf{3.7e-4} & 9.4e-3          & 5.0e-2        \\
$t^2$                 & 1.3e7  & 7.5             & 0.77            & \textbf{0.56} \\
$\sin(t)$             & 6.6e-2 & 1.0e-3          & \textbf{9.4e-4} & 1.2e-2        \\
$\sin(t\sin(t))$      & 2.4    & 0.67            & 0.31            & \textbf{0.16}
\end{tabular}
\caption{The average prediction error across 4096 generations of evolution with population size 16384 and top-1024 selection with 64 seeds. For each experiment we tuned $\beta \in \{1.0, 0.5, 0.1, 0.05, 0.01\}$}.
\label{table:tsf}
\vspace{-1.5em}
\end{table}

\section{Time Series Forecasting}

Next, we consider a time series forecasting task where the goal is to predict the next value of some function $f(t)$. The fitness of an individual $x_t$ is determined by $\textit{fitness}(x_t) = -|f(t/100) - x_t^0|$. We report the results on a number of functions in Table \ref{table:tsf}. Higher orders of meta-parameters improve performance in many of these settings.

\section{Future Work}

One could investigate the emergence of higher-order meta-learning in multi-agent systems \citep{lu2022model, lu2022adversarial} or artificial life \citep{langton1997artificial}. Future work would also involve theoretically analysing the long-term properties of these systems, alongside evaluating other parameterisation and selection schemes on more practical time series forecasting tasks.




\footnotesize
\bibliographystyle{apalike}
\bibliography{example} 

\end{document}